\documentclass[preprint,12pt]{elsarticle}
\usepackage{color}
\usepackage{graphicx}
\usepackage{verbatim}
\usepackage{amsmath,amssymb}
\RequirePackage{lineno}
\usepackage{fancyhdr}
\usepackage{subfigure}
\usepackage{lscape}
\usepackage{verbatim}
\usepackage{amssymb}
\usepackage{amsmath}
\usepackage{amsthm}
\usepackage{color}
\usepackage[ruled]{algorithm2e}
\usepackage{lineno}

\begin{document}
\title{\textrm{Leveraging Elastic Demand for Forecasting}}

\author{Houtao Deng}
\author{Ganesh Krishnan}
\author{Ji Chen}
\author{Dong Liang \footnote{Dong Liang contributed to this work during his employment at Instacart during 1/30/2017 to 6/16/2017.}}
\address{Maplebear Inc. d/b/a Instacart}

\begin{abstract}
Demand variance can result in a mismatch between planned supply and actual demand. Demand shaping strategies such as pricing can be used to reduce the imbalance between supply and demand. In this work, we propose to consider the demand shaping factor in forecasting. We present a method to reallocate the historical elastic demand to reduce variance, thus making forecasting and supply planning more effective. \end{abstract}


\maketitle

\newtheorem{Definition}{Definition}
\newtheorem{Lemma}{Lemma}
\newtheorem{theorem}{Theorem}

\section{Introduction}

Demand forecasting can help businesses estimate future demand and plan supply. Take Instacart as an example. We forecast hourly demand using historical data, as illustrated in Figure \ref{fig:Time_series_original_1}. The following steps can be used to forecast demand for a future hour (e.g., 10-11 am). First, extract the historical demand for the hour. Figure \ref{fig:Time_series_original_10} shows the 10-11 demand in the past 10 days. Then, use a time series model to forecast demand for the hour.

\def\myWidth{5}
\begin{figure}[!]
\centering \includegraphics[width=1.05\linewidth]{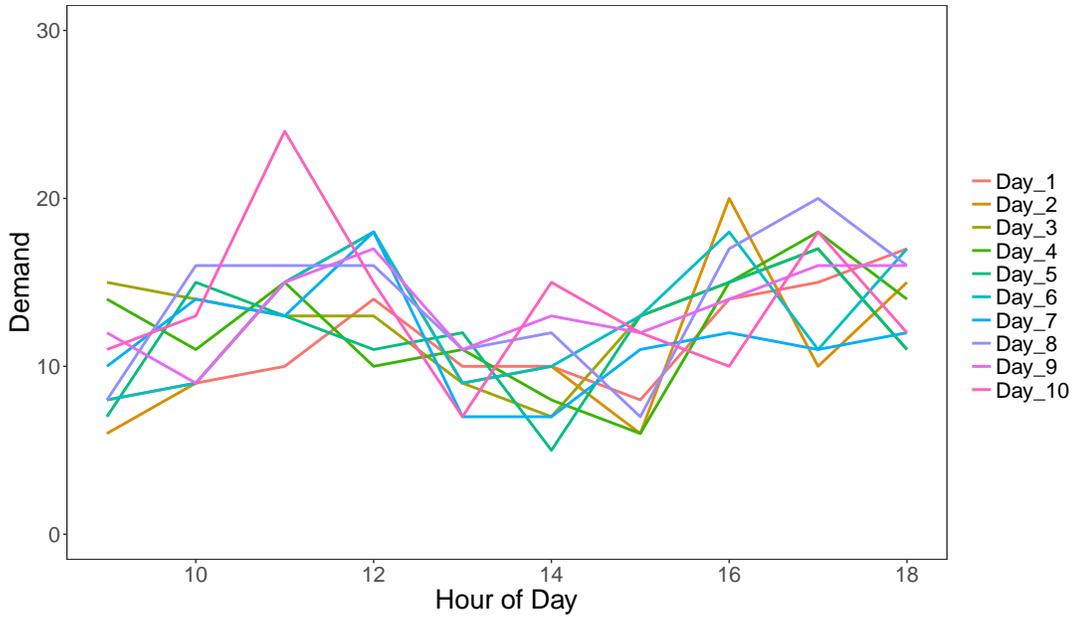}
\caption{Hourly demand from the past 10 days (data are simulated for illustrative purposes). \label{fig:Time_series_original_1}}
\end{figure}

\begin{figure}[!htb]
   \begin{minipage}{0.48\textwidth}
     \centering
     \includegraphics[height=0.6\linewidth]{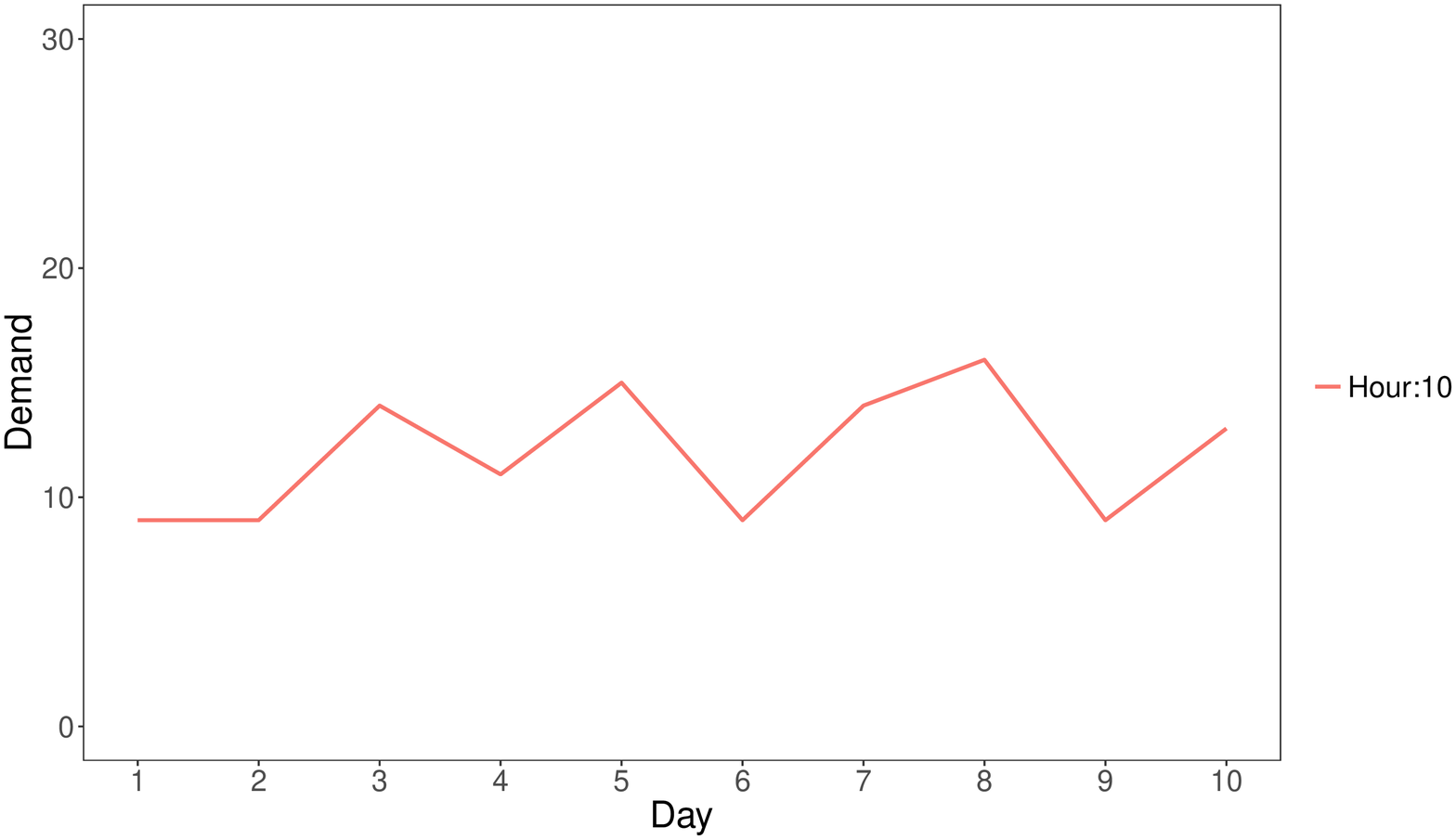}
     \caption{Demand at 10-11 am in the past 10 days. \label{fig:Time_series_original_10}}
   \end{minipage}\hfill
   \begin{minipage}{0.48\textwidth}
     \centering
     \includegraphics[height=0.6\linewidth]{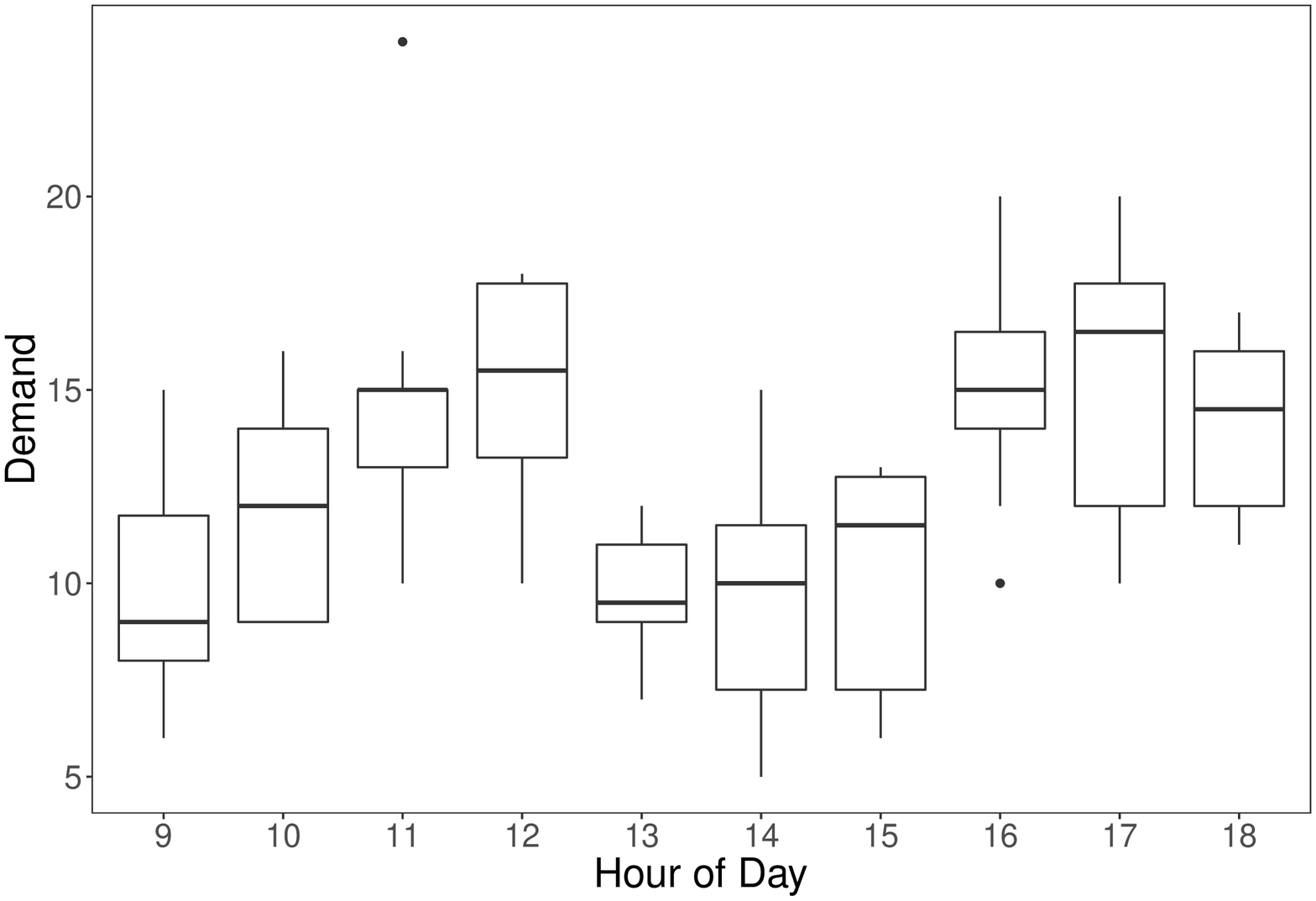}
     \caption{Variability of the demand at each hour (a box spans the first quartile to the third quartile). \label{fig:Box_plot_var}}
   \end{minipage}
\end{figure}

Although time series models can capture trend and seasonality in demand, there can still be a large unexplained variance making forecasting challenging. For the 10-11 demand time series shown in Figure \ref{fig:Time_series_original_10}, there is no obvious seasonality or trend, and the unexplained variance is high. Figure \ref{fig:Box_plot_var} shows a box plot demonstrating the demand variability at each hour from the data in Figure \ref{fig:Time_series_original_10}.

The impact of variance on forecasting and supply planning is illustrated as follows. Table \ref{tab:patter12} describes two historical demand patterns (each with a probability of 50\%). Consider two forecasts, one using the average of the two patterns at each hour, and another just using pattern 1.

\begin{table}[h]
\begin{center}
\begin{tabular}{ccccc}
\hline
Hour & Pattern 1 & Pattern 2 & Forecast 1 & Forecast 2 \\
\hline
9 - 10 & 10 & 20 & 15 & 10 \\
10 - 11 & 20 & 10 & 15 & 20 \\
\hline
\end{tabular}
\end{center}
\caption{ Historical demand patterns each with a probability of 50\%. \label{tab:patter12}}
\end{table}

Assume one unit of supply can serve one unit of demand, and the cost of losing one unit of demand or holding one unit of excessive supply is 1. The expected costs for the two forecasts are shown in Table \ref{tab:patter12cost}. Both forecasts cost 10 in total. The variance makes a zero-cost forecast seemingly impossible.

In this work, we propose to consider elastic demand (demand responsive to demand shaping tactics) in forecasting so that the variance can be reduced.

\begin{table}[h]
\begin{center}
\scalebox{0.8}{
\begin{tabular}{cccccc}
\hline
& Hour & Forecast & Pattern 1 (cost) & Pattern 2 (cost) & Expected cost \\
\hline
Forecast 1 & 9 - 10 & 15 & 10 (5) & 20 (5) & 0.5*5+0.5*5=5 \\
& 10 - 11 & 15 & 20 (5) & 10 (5) & 0.5*5+0.5*5=5 \\
\hline
Forecast 2 & 9 - 10 & 10 & 10 (0) & 20 (10) & 0.5*0+0.5*10=5 \\
& 10 - 11 & 20 & 20 (0) & 10 (10) & 0.5*0+0.5*10=5 \\
\hline
\end{tabular}
}
\end{center}
\caption{ The expected costs of two forecasts. \label{tab:patter12cost}}
\end{table}

\def\myWidth{2}
\begin{figure}[!]
\centering \includegraphics[width= \myWidth in]{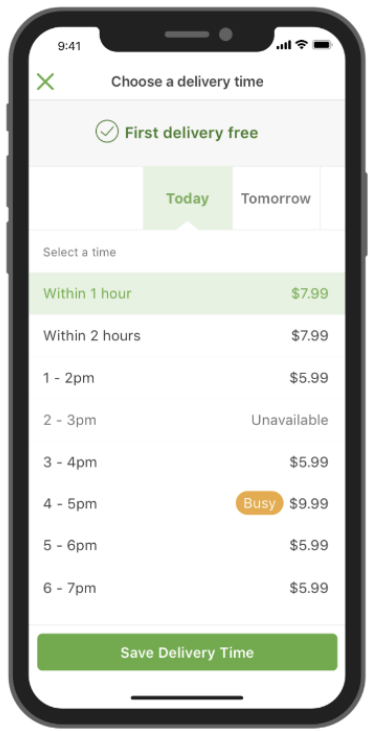}
\caption{Illustration of the availability and pricing information of the delivery windows on Instacart. \label{fig:availability}}
\end{figure}

\section{Related work}
Demand variance poses challenges in supply chain management. Failure to account for demand variance could either lead to unsatisfied customer demand translating to a loss of market share or excessively high inventory holding costs \citep{petkov1997multiperiod,gupta2003managing}.

Strategies have been proposed for reducing demand variance. In a context where orders occur at fixed intervals, suppliers' demand variance will generally decline as the customers' order interval is lengthened \citep{cachon1999managing}. For a manufacturing process consisting of multiple stages, it was observed that reversing two consecutive stages of a process could lead to variance reduction \citep{lee1998variability}. These strategies may require substantial changes to an existing product or process.

Demand shaping strategies such as pricing have been used for reducing demand variance \citep{zhu2016managing}, but most work focuses on the demand side. In this work, we consider elastic demand in the forecasting process, so that the variance in demand can be reduced, and supply planning can be more effective.

\section{Elastic Demand}

In some situations, certain customers are flexible on the product/service options, and so demand shaping strategies such as pricing can influence their choices. The demand responsive to demand shaping tactics is referred to as the elastic demand. Note some historical orders were influenced by demand shaping. Based on the availability and pricing information customers saw (illustrated in Figure \ref{fig:availability}), we can infer the choices customers would have made without demand shaping.

For the example shown in Figure \ref{fig:availability}, the ``within 1 hour" and ``within 2 hours" options are more expensive than the ``1-2pm" option. We also know customers typically prefer faster deliveries. Therefore, for customers placing orders for 1-2pm, there is a high chance that they could have chosen ``within 1 hour" given the same price.

In the rest of this work, we assume the amount of elastic demand is known and focus on how to reallocate it.

For the example in Table 1, assume 10 units of demand at 9-10 in pattern 2 is elastic and can be reallocated to 10-11. As illustrated in Figure \ref{fig:shapable_demand_bars}, the new demand can be any integer value between 10 and 20 at 9-10, and between 10 and 20 at 10-11, respectively.

\def\myWidth{3}
\begin{figure}[!]
\centering \includegraphics[width= \myWidth in]{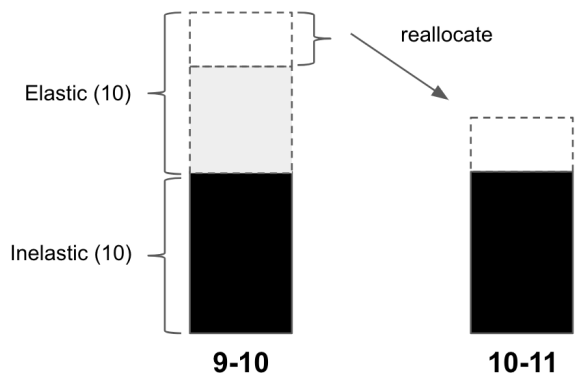}
\caption{9-10 has 10 units of elastic demand and a subset (0\%-100\%) can be reallocated to 10-11. \label{fig:shapable_demand_bars}}
\end{figure}

With elastic demand considered, both forecasts have smaller expected costs and forecast 2 has a zero cost, as shown in Table 3.

\begin{table}[h]
\begin{center}
\scalebox{0.8}{
\begin{tabular}{cccccc}
\hline
& Hour & Forecast & Pattern 1 (cost) & Pattern 2 New (cost) & Expected cost \\
\hline
Forecast 1 & 9 - 10 & 15 & 10 (5) & [20,19,...,10] (0) & 0.5*5+0.5*0=2.5 \\
& 10 - 11 & 15 & 20 (5) & [10,11,...,20] (0) & 0.5*5+0.5*0=2.5 \\
\hline
Forecast 2 & 9 - 10 & 10 & 10 (0) & [20,19,...,10] (0) & 0 \\
& 10 - 11 & 20 & 30 (0) & [10,11,...,20] (0) & 0 \\
\hline
\end{tabular}
}
\end{center}
\caption{ Expected costs of the two forecasts considering Pattern 2 has 10 units of elastic demand. \label{tab:patter12newcost}}
\end{table}

In this article, we propose to consider the elastic demand (demand responsive to demand shaping) in the forecasting process. We present a method to reallocate the elastic demand in historical data so that the variance of the optimized demand is minimized, leading to more effective forecasts.

\section{Elastic demand optimization}

Consider an illustrative example first. Figure \ref{fig:two_series_merge_demo} (left) shows two time series (referred to as demand series) each with 3 time slots. We want the difference between the demand series to be small.
Assuming each time slot of the demand series has 10 units of elastic demand that can be shifted to its adjacent time slots, we can reallocate the demand as follows so that the two series become identical. First shift 5 units from T1 to T2, resulting in the series shown in Figure \ref{fig:two_series_merge_demo} (middle). Then, shift 5 units from T2 to T3, resulting in two identical series, as shown in Figure \ref{fig:two_series_merge_demo} (right).

\def\myWidth{5}
\begin{figure}[h]
  \centering 
  \includegraphics[width=10cm]{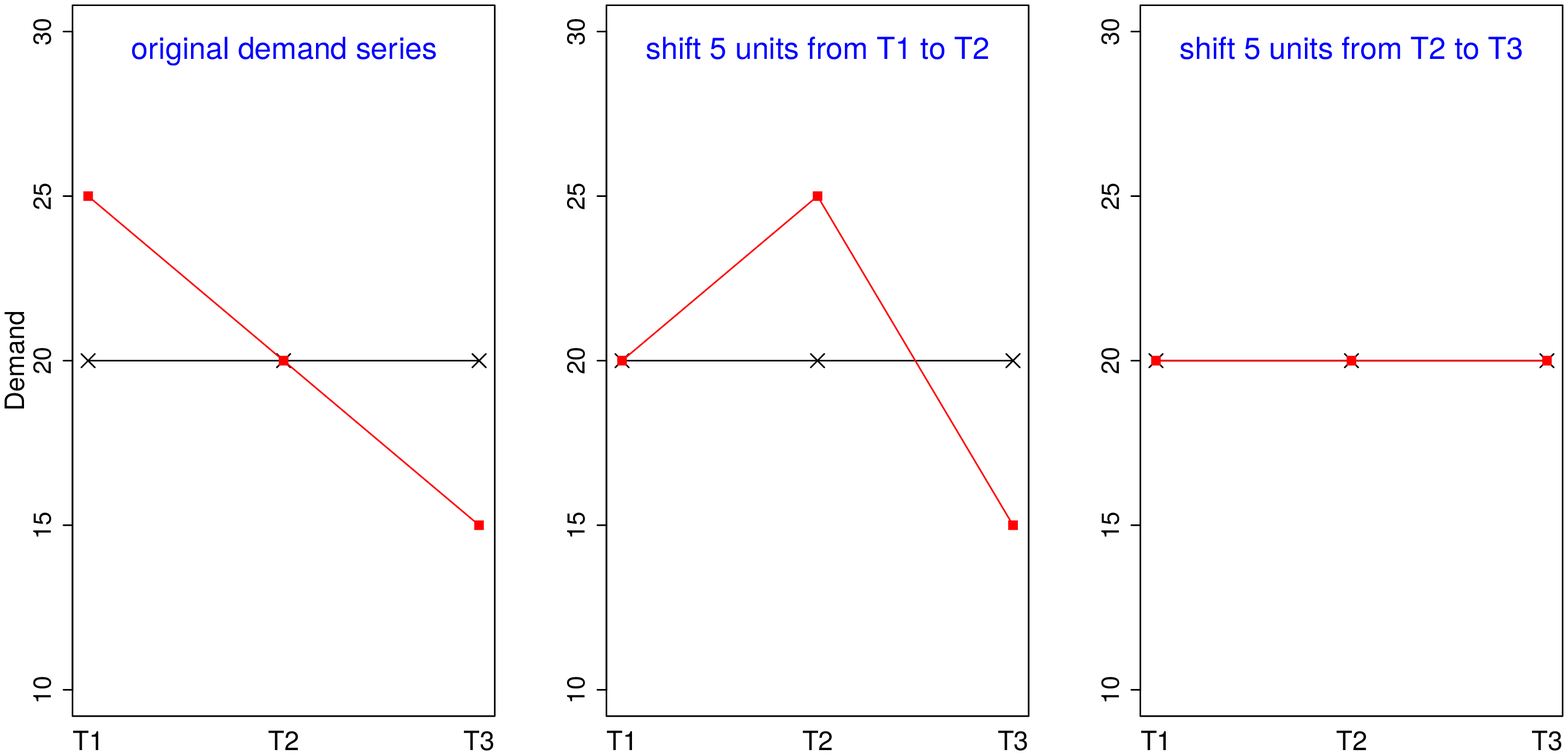}
  \caption{Reallocate elastic demand to make the two demand series identical.}
  \label{fig:two_series_merge_demo}
\end{figure}

From this example, we can generalize the goal, input, and output as follows.

\begin{itemize}
  \item Goal: shifting the right amount of elastic demand to minimize the variance between the new demand series.
  \item Input: historical demand series, and the amount of elastic demand.
  \item Output: shifted demand series
\end{itemize}

We propose the following formulation with a limitation that elastic demand can only be shifted between adjacent time slots.

Consider $K$ historical demand series, each including $T$ time slots. One example of such a series is hourly demand per day. Let $D_{k,t}$ denote the observed demand at time $t$ of the $k^{th}$ series. The goal is to forecast hourly demand $\tilde{D}_{t}$ based on the $K$ demand series data. Assuming demand at different time slots is independent, the mean of $D_{k,t}$ where $k$ = 1, \ldots, $K$ for a given $t$ is the solution in terms of minimizing the sum of squared error

\begin{equation*}
\begin{aligned}
& \underset{\tilde{D}_{t}}{\text{minimize}}
& &  ( D_{k,t} - \tilde{D}_{t} ) ^2
\end{aligned}
\end{equation*}

Consider a certain percentage of demand at $t$ of the $k^{th}$ series is elastic and shiftable between adjacent time slots. Let $x_{k,t,t+1}$ denote the percentage of $D_{k,t}$ that are shifted between $D_{k,t}$ and $D_{k,t+1}$. A positive $x_{k,t,t+1}$ indicates demand  shifted from $t$ to $t+1$, and a negative value indicates demand shifted from $t+1$ to $t$. The goal is to find $x_{k,t,t+1}$ (for all $k$ and $t$) to minimize the variance of the shifted demand.

\begin{equation*}
\begin{aligned}
& {\text{minimize}}
& &  \sum_{k,t} \{ D_{k,t-1} \cdot x_{k,t-1,t} + D_{x,t} \cdot (1-x_{k,t,t+1}) - \tilde{D}_{t} \} ^2 \\
& \text{subject to}
& & x_{k,t,t+1} \in [L_{k,t,t+1},U_{k,t,t+1}], \; t = 1, \ldots, T \\
& & & \tilde{D}_{t} \geq 0, t = 1, \ldots, T
\end{aligned}
\end{equation*}

where $L_{k,t,t+1}$ and $U_{k,t,t+1}$ are the lower and upper limits of $x_{k,t,t+1}$. $U (>=0)$ is essentially the percentage of elastic demand that can be shifted from $t$ to $t+1$, and $L (<=0)$ is the percentage of elastic demand shiftable from $t+1$ to $t$. Also, let $D_{k,0} = 0$, $x_{k,0,1} = 0$ and $x_{k,T,T+1}=0$.

In addition, when two solutions have similar variance, the one with a smaller $x_{k,t,t+1}$ may be preferred. To this end, we add a regularization term.

\begin{equation*}
\begin{aligned}
& {\text{minimize}}
& &  \sum_{k,t} \{ D_{k,t-1} \cdot x_{k,t-1,t} + D_{x,t} \cdot (1-x_{k,t,t+1}) - \tilde{D}_{t} \} ^2 + \lambda \cdot \sum_{k,t}x_{k,t,t+1}^2     \\
& \text{subject to}
& & x_{k,t,t+1} \in [L_{k,t,t+1},U_{k,t,t+1}], \; t = 1, \ldots, T \\
& & & \tilde{D}_{t} \geq 0, t = 1, \ldots, T
\end{aligned}
\end{equation*}

This is a convex optimization problem which means an optimal solution can be found. In addition, the same property holds with changes such as, using different $\lambda$s for different $x_{k,t,t+1}$, and regularizing the amount of shifted demand ($D_{x,t} * x_{k,t,t-1}$) instead of the percentage.

Let $\overline{D}_{k,t}$ denote the demand after being shifted, i.e., $\overline{D}_{k,t} = D_{k,t-1} \cdot x_{k,t,t-1} + D_{x,t} \cdot (1-x_{k,t,t+1})$. We can prove for any given $k$, the total demand is not changed after being shifted.

\begin{theorem}
For a given $k$, $\sum_{t}\overline{D}_{k,t} = \sum_{t}D_{k,t}$.
\end{theorem}

\begin{proof}
$\sum_{t}\overline{D}_{k,t} = \sum_{t}\{ D_{k,t-1} \cdot x_{k,t-1,t} + D_{x,t} \cdot (1-x_{k,t,t+1})\} = D_{k,0}\cdot x_{k,0,1} + D_{k,1} \cdot (1-x_{k,1,2})
+ D_{k,1}\cdot x_{k,1,2} + D_{k,2} \cdot (1-x_{k,2,3}) + \ldots + D_{k,T-1}\cdot x_{k,T-1,T} + D_{k,T} \cdot (1-x_{k,T,T+1})    $

since $D_{k,0}$, $x_{k,0,1}$, $x_{k,T,T+1}$ are zeros, $\sum_{t}\overline{D}_{k,t} = \sum_{t}D_{k,t}$
\end{proof}

It should be noted that the following linear optimization formulation could achieve the same purpose.

\begin{equation*}
\begin{aligned}
& {\text{minimize}}
& &  \sum_{k,t} | D_{k,t-1} \cdot x_{k,t-1,t} + D_{x,t} \cdot (1-x_{k,t,t+1}) - \tilde{D}_{t} | + \lambda \cdot \sum_{k,t} | x_{k,t,t+1} |     \\
& \text{subject to}
& & x_{k,t,t+1} \in [L_{k,t,t+1},U_{k,t,t+1}], \; t = 1, \ldots, T \\
& & & \tilde{D}_{t} \geq 0, t = 1, \ldots, T
\end{aligned}
\end{equation*}

\section{Experiments}

\begin{table}[]
\scriptsize
\begin{center}
\begin{tabular}{c|cc|c|cc|c|ccc}
\hline
\\
\\[-2em]
& L & U & $\lambda$ & 9-10 & 10-11 & $\overline{var}$ & $p$ & $|p|$ & $\overline{|p|}$ \\ [0.4em]
\hline
Pattern 1 & & & & 10 & 20 & 50 & & & \\
Pattern 2 & & & & 20 & 10 & & & & \\
\hline
${Pattern\ 1}^*$ & -1 & 1 & 0 & 13 & 17 & 0 & -0.3 & 0.3 & 0.325 \\
${Pattern\ 2}^*$ & -1 & 1 & 0 & 13 & 17 & & 0.35 & 0.35 & \\
\hline
${Pattern\ 1}^*$ & 0 & 1 & 0 & 7 & 23 & 0 & 0.3 & 0.3 & 0.475 \\
${Pattern\ 2}^*$ & 0 & 1 & 0 & 7 & 23 & & 0.65 & 0.65 & \\
\hline
${Pattern\ 1}^*$ & -1 & 0 & 0 & 20 & 10 & 0 & -1 & 1 & 0.5 \\
${Pattern\ 2}^*$ & -1 & 0 & 0 & 20 & 10 & & 0 & 0 & \\
\hline
${Pattern\ 1}^*$ & -1 & 1 & 1 & 12 & 18 & 0 & -0.2 & 0.2 & 0.3 \\
${Pattern\ 2}^*$ & -1 & 1 & 1 & 12 & 18 & & 0.4 & 0.4 & \\
\hline
${Pattern\ 1}^*$ & -0.1 & 0.1 & 0 & 11 & 19 & 24.5 & -0.1 & 0.1 & 0.1 \\
${Pattern\ 2}^*$ & -0.1 & 0.1 & 0 & 18 & 12 & & 0.1 & 0.1 & \\
\hline
\end{tabular}
\caption{ Original and shifted demand (marked with $^*$) at different parameters, and the summary statistics. \label{tab:simpleexample}}
\end{center}
\end{table}

Consider the example shown in Table \ref{tab:patter12}. Table \ref{tab:simpleexample} shows original demand, and the shifted demand at different values of $L$ (percentage of elastic demand shiftable to the earlier time slot), $U$ (percentage of elastic demand shiftable to the later time slot) and $\lambda$. The following statistics are calculated: the average variance at each hour, the shifting percentage ($p$) from 9-10 to 10-11 for pattern 1 and pattern 2, the absolute values of $p$ ($|p|$), and the average of $|p|$ ($\overline{|p|}$).

When there is no limit on the amount of elastic demand, i.e., $L=1$ or/and $U=1$, the demand can be shifted so that the two patterns become identical and therefore the average variance is zero. As expected, when the magnitude of shifting percentage is penalized ($\lambda=1$), the average of the magnitude of $p$, $\overline{|p|}$, is reduced. When $U=0.1$ and $L=-0.1$, 10\% of the demand at 9-10 of pattern 1 is shifted from 10-11, and 10\% of the demand at 9-10 of pattern 2 is shifted to 10-11. The average variance is reduced but not zero, indicating the amount of elastic demand is not large enough for a zero variance.

Next, we consider the demand series shown in Figure \ref{fig:Time_series_original_1}. For simplicity, when applying the optimization framework, we let the lower and upper limits of the shifting percentages have the same absolute value, i.e., $U=-L$.

The average variance at different $U$s and $\lambda$s are shown in Figure \ref{fig:Max_p_vs_variance}. The relationship between $\overline{var}$ and $U$ is the same for $\lambda \in (0,1)$ (the two curves overlap). The average variance decreases quickly as $U$ increases from 0 and reaches a small value when $U$ gets close to 0.6. When $\lambda$ increases to 10 or 100, the average variance still decreases initially but remain at a larger value after some points. This is because when $\lambda$ gets large enough, the cost of increased shifting probabilities can exceed the cost of a larger variance.

\begin{figure}[!]
   \begin{minipage}{0.48\textwidth}
     \centering
     \includegraphics[width=1.05\linewidth]{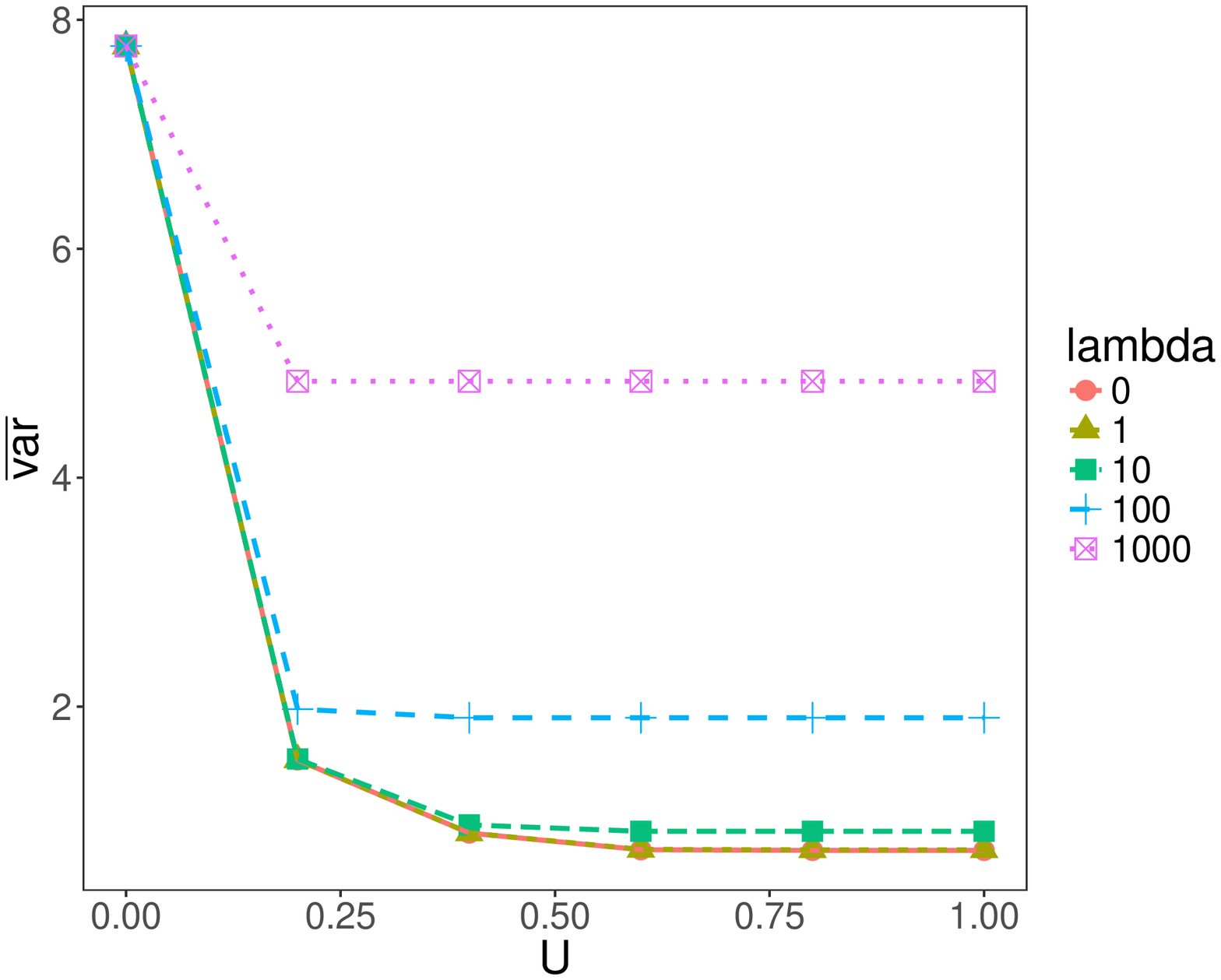}
     \caption{The average variance at different $U$s and $\lambda$s. \label{fig:Max_p_vs_variance}}
   \end{minipage}\hfill
   \begin{minipage}{0.48\textwidth}
     \centering
     \includegraphics[width=1.05\linewidth]{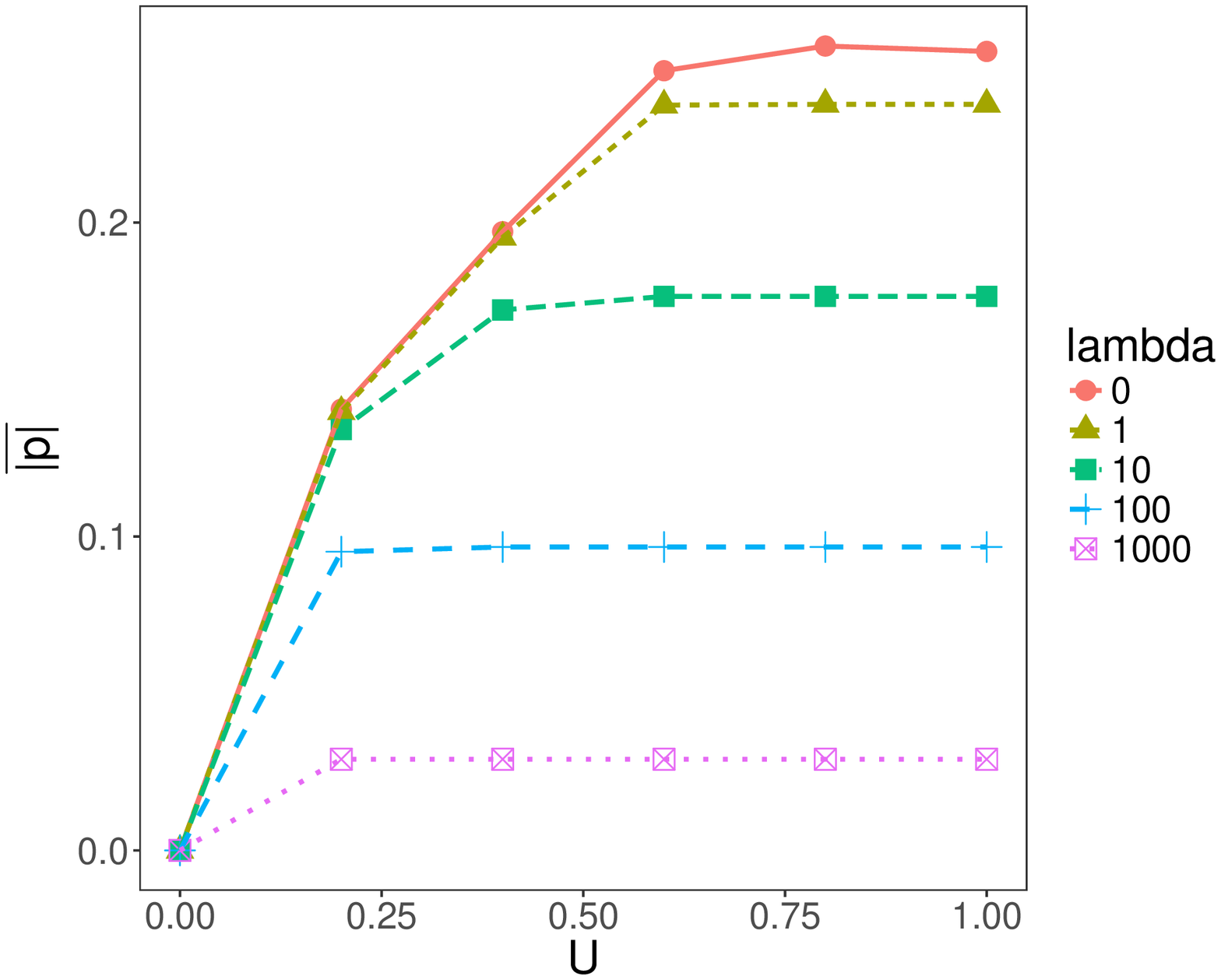}
     \caption{The average magnitude of shifting probabilities at different $U$s and $\lambda$s.\label{fig:Mean_p_vs_max_p} }
   \end{minipage}
\end{figure}

The average magnitudes of shifting probabilities at different $U$s and $\lambda$s are shown in Figure \ref{fig:Mean_p_vs_max_p}. In general, $\overline{|p|}$ increases as $U$ increases. However, with a larger $\lambda$, the growth of $\overline{|p|}$ is smaller.

Note for $\lambda=0$, the average magnitude of shifting probabilities decreases when $U$ increases from 0.8 to 1. This is because when $U$ is large enough, there can be multiple solutions ($p$) leading to the same variance reduction, and with a zero $\lambda$ there is no preference for a samller $p$. For cases with $\lambda>0$, $\overline{|p|}$ becomes at a constant after some point (when there are multiple solutions with the same variance reduction, the solution with a smaller $\sum_{k,t}x_{k,t,t+1}^2$ is preferred).

Figure \ref{fig:Demand_shifted_6_series} shows the demand series optimized with $\lambda=1$ and 0\% (original demand), 10\%, 20\%, 30\%, 40\%, and 50\% elastic demand, respectively. With larger elastic demand percentages, the variance of hourly demand become smaller, and the demand series tend to have clearer patterns.

Figure \ref{fig:Demand_variance_6_series} shows the average of hourly demand variance for each group of demand series in Figure \ref{fig:Demand_shifted_6_series}. As expected, a larger amount of elastic demand leads to a smaller variance. However, the first 10\% elastic demand produces the largest variance reduction.

\begin{figure}
   \begin{minipage}[t]{0.48\textwidth}
     \centering
     \includegraphics[width=1\textwidth]{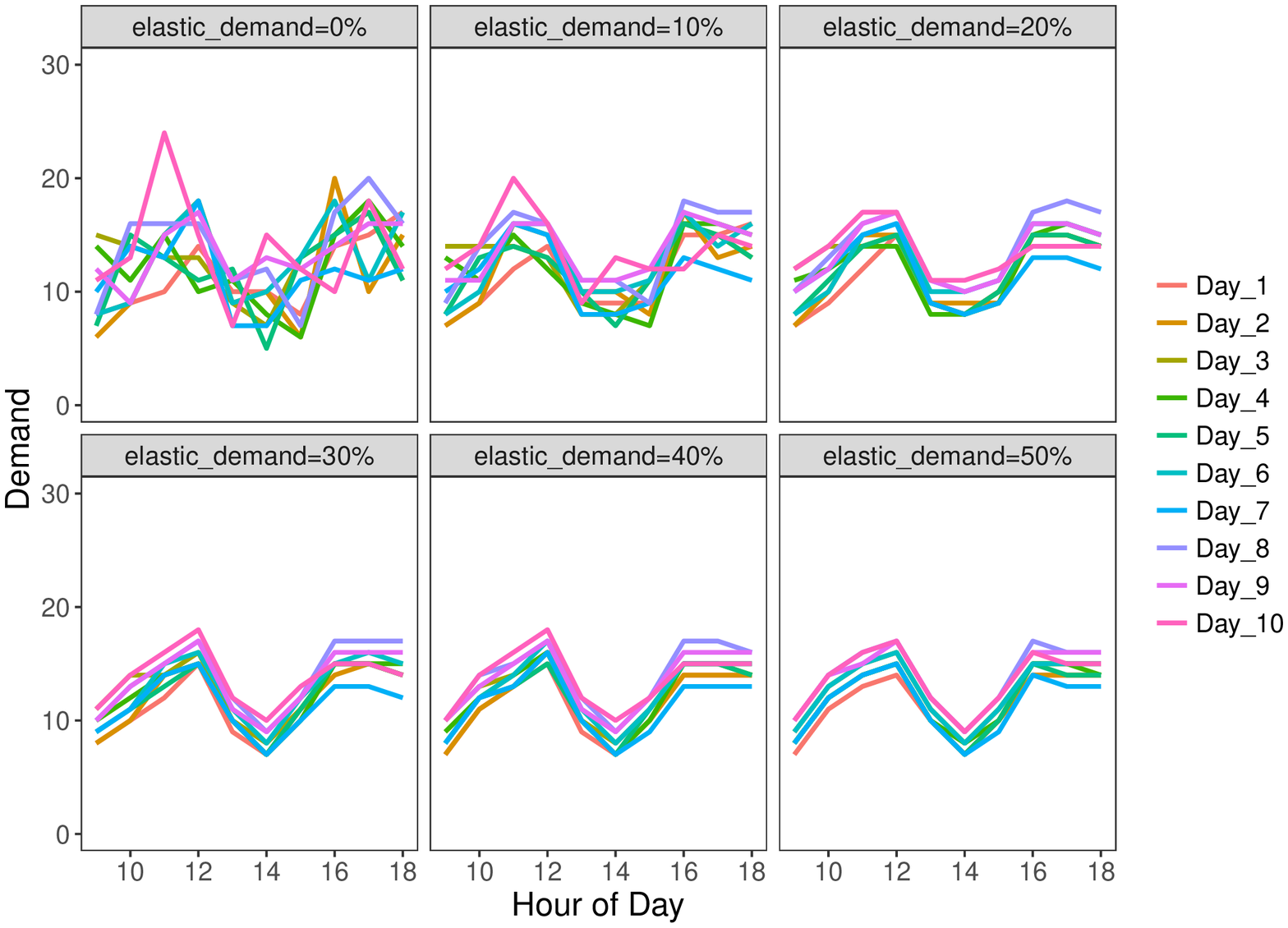}
     \caption{Original demand series and new demand series for different percentages of elastic demand. \label{fig:Demand_shifted_6_series}}
   \end{minipage}\hfill
   \begin{minipage}[t]{0.48\textwidth}
     \centering
     \includegraphics[width=1\textwidth]{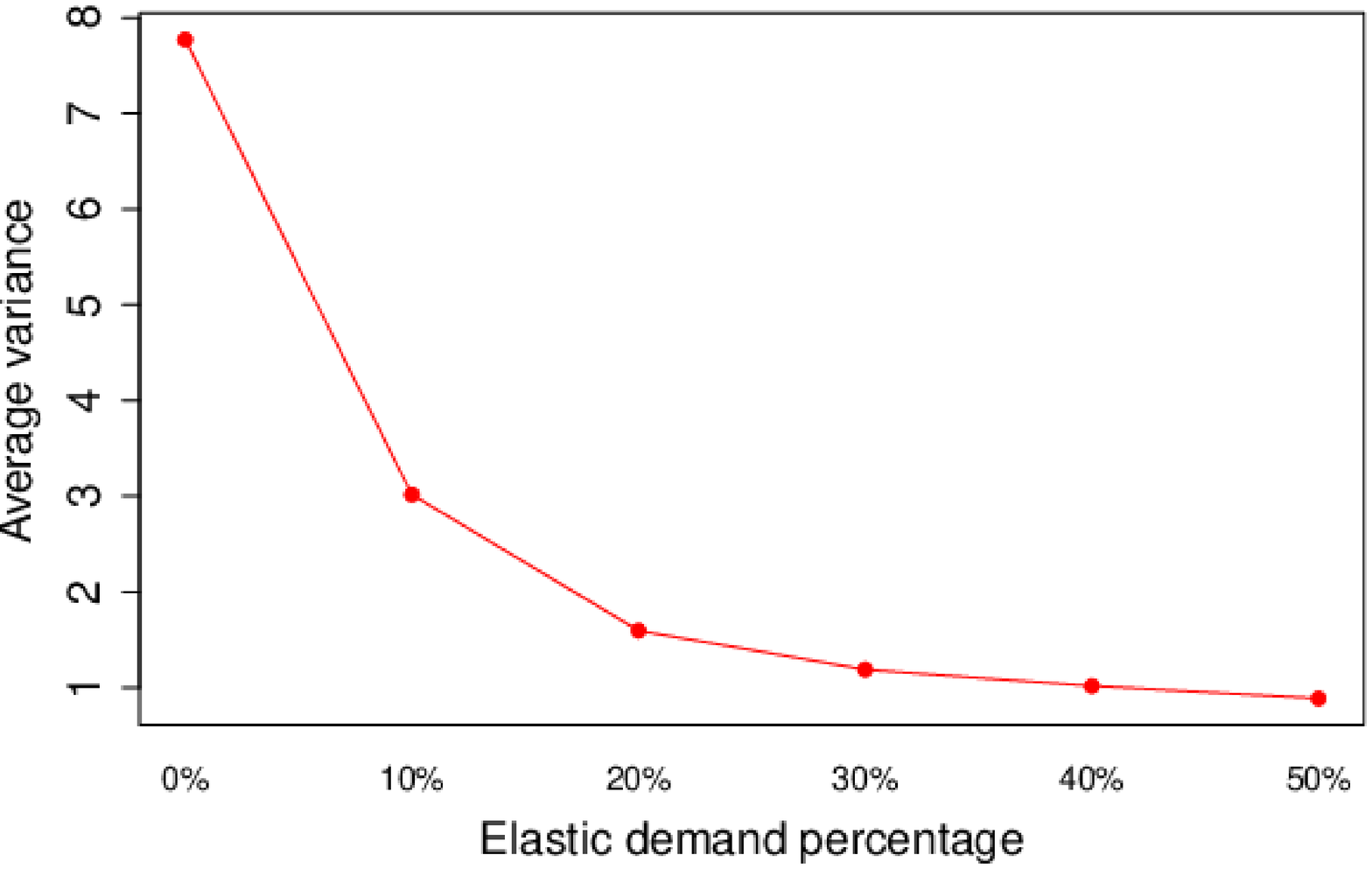}
     \caption{The average of hourly variance for each group of demand series in Figure \ref{fig:Demand_shifted_6_series}. \label{fig:Demand_variance_6_series}}
   \end{minipage}
\end{figure}

\section{Conclusions}
A large demand variance can result in a high cost due to lost sales or excessive supply. In this work, we showed historical elastic demand can be reallocated to reduce the variance, hence making demand forecasting more effective.

As Instacart has hourly delivery windows, we focus on hourly demand with a daily cycle and assume the elastic demand can be reallocated among the adjacent hours from the same day. If seasonality exists in the days of week, the proposed method can be applied to the demand series from different days of week separately.

To extend the method to handle daily demand with a weekly cycle or monthly demand with a yearly cycle, the formulation should be changed to allow the demand to be shiftable between the end of a cycle and the start of next cycle.

\bibliographystyle{natbib}
\bibliography{elastic_demand}

\begin{thebibliography}{}

\bibitem[Cachon(1999)Cachon]{cachon1999managing}
Cachon, G.~P. (1999).
\newblock Managing supply chain demand variability with scheduled ordering
  policies.
\newblock {\em Management science\/}, {\bf 45}(6), 843--856.

\bibitem[Gupta and Maranas(2003)Gupta and Maranas]{gupta2003managing}
Gupta, A. and Maranas, C.~D. (2003).
\newblock Managing demand uncertainty in supply chain planning.
\newblock {\em Computers \& Chemical Engineering\/}, {\bf 27}(8), 1219--1227.

\bibitem[Lee and Tang(1998)Lee and Tang]{lee1998variability}
Lee, H.~L. and Tang, C.~S. (1998).
\newblock Variability reduction through operations reversal.
\newblock {\em Management Science\/}, {\bf 44}(2), 162--172.

\bibitem[Petkov and Maranas(1997)Petkov and Maranas]{petkov1997multiperiod}
Petkov, S.~B. and Maranas, C.~D. (1997).
\newblock Multiperiod planning and scheduling of multiproduct batch plants
  under demand uncertainty.
\newblock {\em Industrial \& engineering chemistry research\/}, {\bf 36}(11),
  4864--4881.

\bibitem[Zhu {\em et~al.}(2016)Zhu, Vaghefi, Jafari, Lu, and
  Ghofrani]{zhu2016managing}
Zhu, J., Vaghefi, S.~A., Jafari, M.~A., Lu, Y., and Ghofrani, A. (2016).
\newblock Managing demand uncertainty with cost-for-deviation retail pricing.
\newblock {\em Energy and Buildings\/}, {\bf 118}, 46--56.

\end{thebibliography}

\end{document}